\documentclass[11pt]{article}
\pdfoutput=1

\RequirePackage[l2tabu, orthodox]{nag}
\usepackage[T1]{fontenc}

\usepackage[bitstream-charter]{mathdesign}
\usepackage{amsmath}
\usepackage[scaled=0.92]{PTSans}

\usepackage[
  paper  = letterpaper,
  left   = 1.65in,
  right  = 1.65in,
  top    = 1.0in,
  bottom = 1.0in,
  ]{geometry}

\usepackage[usenames,dvipsnames]{xcolor}
\definecolor{shadecolor}{gray}{0.9}

\usepackage[final,expansion=alltext]{microtype}
\usepackage[english]{babel}
\usepackage[parfill]{parskip}
\usepackage{afterpage}
\usepackage{framed}

\DeclareRobustCommand{\parhead}[1]{\textbf{#1}~}

\usepackage{lineno}

\usepackage{ragged2e}

\newcounter{parcount}

\usepackage{graphicx}
\usepackage[labelfont=bf]{caption}
\usepackage{subcaption}

\usepackage{natbib}

\usepackage{booktabs}
\usepackage{multirow}

\usepackage[algoruled]{algorithm2e}
\usepackage{listings}
\usepackage{fancyvrb}
\fvset{fontsize=\normalsize}

\usepackage[colorlinks,linktoc=all]{hyperref}
\usepackage[all]{hypcap}
\hypersetup{citecolor=BurntOrange}
\hypersetup{linkcolor=black}
\hypersetup{urlcolor=BurntOrange}

\usepackage[acronym,nowarn]{glossaries}

\lstdefinestyle{mystyle}{
    commentstyle=\color{OliveGreen},
    keywordstyle=\color{BurntOrange},
    numberstyle=\tiny\color{black!60},
    stringstyle=\color{MidnightBlue},
    basicstyle=\ttfamily,
    breakatwhitespace=false,
    breaklines=true,
    captionpos=b,
    keepspaces=true,
    numbers=left,
    numbersep=5pt,
    showspaces=false,
    showstringspaces=false,
    showtabs=false,
    tabsize=2
}
\usepackage[colorinlistoftodos,
           prependcaption,
           textsize=small,
           backgroundcolor=yellow,
           linecolor=lightgray,
           bordercolor=lightgray]{todonotes}

\DeclareRobustCommand{\parhead}[1]{\textbf{#1}~}

\usepackage{amsthm}  

\usepackage{bm}

\usepackage{wrapfig}

\usepackage{booktabs}
\usepackage{arydshln} \usepackage{multirow}

\usepackage{listings}
\lstdefinestyle{alp_style}{
    commentstyle=\color{OliveGreen},
    numberstyle=\tiny\color{black!60},
    stringstyle=\color{BrickRed},
    basicstyle=\ttfamily\scriptsize,
    breakatwhitespace=false,
    breaklines=true,
    captionpos=b,
    keepspaces=true,
    numbers=none,
    numbersep=5pt,
    showspaces=false,
    showstringspaces=false,
    showtabs=false,
    tabsize=2
}

\usepackage{capt-of}

\newcommand{\ourmetric}{{\text{VS}}}

\usepackage{lipsum}

\newtheorem{definition}{Definition}[section]
\newtheorem{theorem}{Theorem}[section]

\newtheorem{lemma}{Lemma}[section]   

\theoremstyle{remark}
\newtheorem*{lemma*}{Lemma}

\def\eqref#1{equation~\ref{#1}}

\def\1{\bm{1}}

\def\vp{{\bm{p}}}

\def\vu{{\bm{u}}}

\def\mD{{\bm{D}}}

\def\mK{{\bm{K}}}
\def\mL{{\bm{L}}}

\def\mU{{\bm{U}}}

\def\mX{{\bm{X}}}

\def\mLambda{{\bm{\Lambda}}}
\def\mSigma{{\bm{\Sigma}}}

\DeclareMathAlphabet{\mathsfit}{\encodingdefault}{\sfdefault}{m}{sl}
\SetMathAlphabet{\mathsfit}{bold}{\encodingdefault}{\sfdefault}{bx}{n}

\def\gS{{\mathcal{S}}}

\def\gX{{\mathcal{X}}}

\def\sR{{\mathbb{R}}}

\DeclareMathOperator{\tr}{tr}

 \newacronym{ALI}{ali}{adversarially learned inference}
\newacronym{BIGAN}{bigan}{bidirectional generative adversarial network}
\newacronym{VI}{vi}{variational inference}
\newacronym{KL}{kl}{Kullback-Leibler}
\newacronym{ELBO}{elbo}{evidence lower bound}
\newacronym{MCMC}{mcmc}{Markov chain Monte Carlo}
\newacronym{HMC}{hmc}{Hamiltonian Monte Carlo}
\newacronym{RNN}{rnn}{recurrent neural network}
\newacronym{MLP}{mlp}{feed forward neural network}
\newacronym{GAN}{gan}{generative adversarial network}
\newacronym{DCGAN}{dcgan}{deep convolutional generative adversarial network}
\newacronym{PresGAN}{presgan}{prescribed generative adversarial network}
\newacronym{DGM}{dgm}{deep generative model}
\newacronym{PGAN}{pgan}{prescribed generative adversarial network}
\newacronym{VEEGAN}{veegan}{vee {GAN}}
\newacronym{PACGAN}{pacgan}{packed {GAN}}
\newacronym{STYLEGAN}{stylegan}{Style {GAN}}
\newacronym{FID}{fid}{{F}r\'{e}chet {I}nception distance}
\newacronym{IS}{is}{{I}nception score}
\newacronym{ML}{ml}{machine learning}
\newacronym{VS}{vs}{vendi score}
\newacronym{NLP}{nlp}{natural language processing}
\newacronym{IntDiv}{intdiv}{{I}nternal {D}iversity}
\newacronym{BLEU}{bleu}{BLEU}
\newacronym{PAIRWISE-BLEU}{pairwise-bleu}{PAIRWISE-BLEU}
\newacronym{D-LEX-SIM}{d-lex-sim}{D-LEX-SIM}
\newacronym{GILBO}{gilbo}{GILBO}
\newacronym{NOM}{nom}{number of modes}
\newacronym{MOSES}{moses}{MOSES}
\newacronym{HMM}{hmm}{HMM}
\newacronym{AAE}{aae}{AAE}
\newacronym{VAE}{vae}{VAE}
\newacronym{JTN}{jtn}{JTN}
\newacronym{Char-RNN}{char-rnn}{Char-RNN}
\newacronym{SMILES}{smiles}{SMILES}
\newacronym{MNIST}{mnist}{MNIST}
\newacronym{MultiNLI}{multinli}{MultiNLI}
\newacronym{StackedMNIST}{stackedmnist}{StackedMNIST}
\newacronym{NLI}{nli}{NLI}
\newacronym{VDVAE}{vdvae}{VDVAE}
\newacronym{LSUN}{lsun}{LSUN}
\newacronym{CIFAR}{cifar}{CIFAR}
\newacronym{ADM}{adm}{ADM}
\newacronym{RQ-VT}{rq-vt}{RQ-VT}
\newacronym{IDDPM}{iddpm}{IDDPM}
\newacronym{DBS}{dbs}{diverse beam search}
\newacronym{MS COCO}{ms coco}{MS COCO}
\newacronym{CelebA}{celeba}{CELEBA}
\newacronym{CIFAR-100}{cifar-100}{CIFAR-100} 
\newacronym{JSD}{jsd}{Jensen-Shannon Divergence}
\newacronym{QD}{qd}{quality diversity}
\newacronym{DPP}{dpp}{{D}eterminantal {P}oint {P}rocess}

\usepackage{authblk}
\usepackage{enumitem}

\title{\textbf{The Vendi Score: A Diversity Evaluation Metric for Machine Learning}}

\author[1]{Dan Friedman}
\author[1, 2, *]{Adji Bousso Dieng}
\affil[1]{Department of Computer Science, Princeton University}
\affil[2]{\href{https://vertaix.princeton.edu/}{Vertaix}}
\affil[*]{\emph{Published in Transactions on Machine Learning Research (07/2023)\\Reviewed on OpenReview \url{https://openreview.net/forum?id=g97OHbQyk1}}}

\begin{document}
\maketitle

\begin{abstract}
\noindent Diversity is an important criterion for many areas of \gls{ML}, including generative modeling and dataset curation.
  However, existing metrics for measuring diversity are often domain-specific and limited in flexibility.
  In this paper, we address the diversity evaluation problem by proposing the \emph{Vendi Score}, which connects and extends ideas from ecology and quantum statistical mechanics to \gls{ML}. The Vendi Score is defined as the exponential of the Shannon entropy of the eigenvalues of a similarity matrix. This matrix is induced by a user-defined similarity function applied to the sample to be evaluated for diversity. In taking a similarity function as input, the Vendi Score enables its user to specify any desired form of diversity. Importantly, unlike many existing metrics in \gls{ML}, the Vendi Score does not require a reference dataset or distribution over samples or labels, it is therefore general and applicable to any generative model, decoding algorithm, and dataset from any domain where similarity can be defined. We showcase the Vendi Score on molecular generative modeling where we found it addresses shortcomings of the current diversity metric of choice in that domain. We also applied the Vendi Score to generative models of images and decoding algorithms of text where we found it confirms known results about diversity in those domains. Furthermore, we used the Vendi Score to measure mode collapse, a known shortcoming of \glspl{GAN}. In particular, the Vendi Score revealed that even \glspl{GAN} that capture all the modes of a labelled dataset can be less diverse than the original dataset. Finally, the interpretability of the Vendi Score allowed us to diagnose several benchmark \gls{ML} datasets for diversity, opening the door for diversity-informed data augmentation\footnote{Code for calculating the Vendi Score is available at \url{https://github.com/vertaix/Vendi-Score}.}.\\
  
\end{abstract}

\section{Introduction}
\glsresetall

\begin{figure*}[!ht]
     \centering
     \begin{subfigure}[c]{0.05\textwidth}
         \centering
         \caption{}
         \label{fig:effective_number}
     \end{subfigure}\hfill
     \begin{subfigure}[c]{0.95\textwidth}
         \centering
         \includegraphics[height=0.14\paperheight]{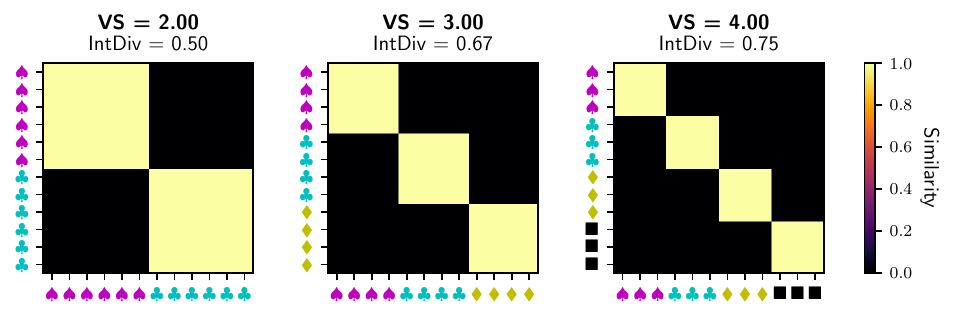}
     \end{subfigure}
     \begin{subfigure}[c]{0.05\textwidth}
         \centering
         \caption{}
         \label{fig:nonlinear}
     \end{subfigure}\hfill
     \begin{subfigure}[c]{0.95\textwidth}
         \centering
         \includegraphics[height=0.14\paperheight]{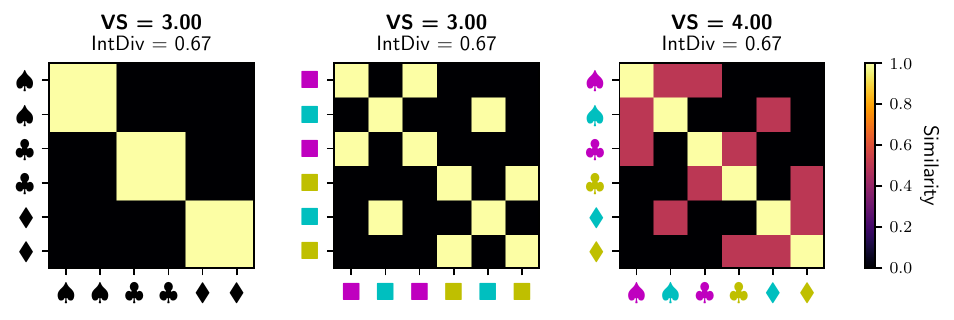}
     \end{subfigure}
     \begin{subfigure}[c]{0.05\textwidth}
         \centering
         \caption{}
         \label{fig:correlations}
     \end{subfigure}\hfill
     \begin{subfigure}[c]{0.95\textwidth}
         \centering
         \includegraphics[height=0.14\paperheight]{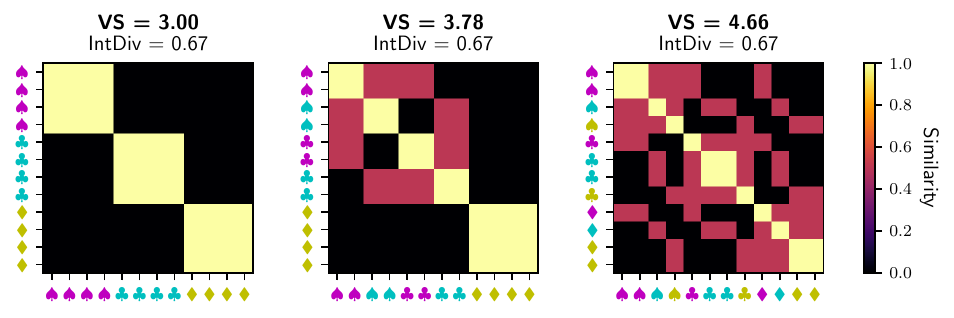}
     \end{subfigure}
     \caption{
       (a) The Vendi Score, \acrshort{VS} in the figure, can be interpreted as the effective number of unique elements in a sample. It increases linearly with the number of modes in the dataset. IntDiv, the expected dissimilarity, becomes less sensitive as the number of modes increases, converging to 1.
      (b) Combining distinct similarity functions can increase the Vendi Score, as should be expected of a diversity metric, while leaving IntDiv unchanged.
      (c) IntDiv does not take into account correlations between features, but the Vendi Score does. The Vendi Score is highest when the items in the sample differ in many attributes, and the attributes are not correlated with each other.
     }
        \label{fig:properties}
\end{figure*}
 
Diversity is a criterion that is sought after in many areas of \gls{ML}, from dataset curation 
and generative modeling to reinforcement learning, active learning, and decoding algorithms. A lack of diversity in datasets and 
models can hinder the usefulness of \gls{ML} in many critical applications, e.g. scientific discovery.   
It is therefore important to be able to measure diversity.

Many diversity metrics have been proposed in ML, but these metrics are often domain-specific and limited in flexibility.
These include metrics that define diversity in terms of a reference dataset~\citep{heusel2017gans,sajjadi2018assessing}, a pre-trained classifier~\citep{Salimans2016,Srivastava2017}, or discrete features, like n-grams~\citep{li2016diversity}.
In this paper, we propose a general, reference-free approach that defines diversity in terms of a user-specified similarity function.

Our approach is based on work in ecology, where biological diversity has been defined as the exponential of the entropy of the distribution of species within a population~\citep{hill1973diversity,jost2006entropy,leinster2021entropy}.
This value can be interpreted as the effective number of species in the population.
To adapt this approach to ML, we define the diversity of a collection of elements $x_1, \ldots, x_n$ as the exponential of the entropy of the eigenvalues of the $n \times n$ similarity matrix $\mK$, whose entries are equal to the similarity scores between each pair of elements.
This entropy can be seen as the von Neumann entropy associated with $\mK$~\citep{bengtsson2017geometry}, so we call our metric the \textit{Vendi Score}, for the von Neumann diversity.

\parhead{Contributions.} 
We summarize our contributions as follows:
\begin{itemize}
	\item We extend ecological diversity to ML, and propose the Vendi Score, a metric for evaluating diversity in ML. 
	We study the properties of the Vendi Score, which provides us with a more formal understanding of desiderata for diversity.  
	\item We showcase the flexibility and wide applicability of the Vendi Score, characteristics that stem from its sole reliance on the sample to be evaluated for diversity and a user-defined similarity function, and highlight the shortcomings of existing metrics used to measure diversity in different domains. 
\end{itemize}

\section{Are We Measuring Diversity Correctly in ML?}

Several existing metrics for diversity rely on a reference distribution or dataset. 
These reference-based metrics define diversity in terms of coverage of the reference.
They assume access to an embedding function--such as a pretrained Inception model~\citep{szegedy2016rethinking}--that maps samples to real-valued vectors.
One example of a reference-based metric is \gls{FID}~\citep{heusel2017gans}, which measures the Wasserstein-2 distance between two Gaussian distributions, one Gaussian fit to the embeddings of the reference sample and another one fit to the embeddings of the sample to be evaluated for diversity.
\gls{FID} was originally proposed for evaluating image \glspl{GAN} but has since been applied to text~\citep{cifka2018eval} and molecules~\citep{preuer2018frechet} using domain-specific neural network encoders.
~\citet{sajjadi2018assessing} proposed a two-metric evaluation paradigm using precision and recall, with precision measuring quality and recall measuring diversity in terms of coverage of the reference distribution. Several other variations of precision and recall have been proposed~\citep{kynkaanniemi2019improved,simon2019revisiting,naeem2020reliable}.  
Compared to these approaches, the Vendi Score is a reference-free metric, measuring the intrinsic diversity of a set rather than the relationship to a reference distribution. This means that the Vendi Score should be used along side a quality metric, but can be applied in settings where there is no reference distribution.

Some other existing metrics evaluate diversity using a pre-trained classifier, therefore requiring labeled datasets.
For example, the \gls{IS}~\citep{Salimans2016}, which is mainly used to evaluate the perceptual quality of image generative models, 
evaluates diversity using the entropy of the marginal distribution of class labels predicted by an ImageNet classifier.
Another example is \gls{NOM}~\citep{Srivastava2017}, a metric used to evaluate the diversity of \glspl{GAN}. 
\gls{NOM} is calculated by using a classifier trained on a labeled dataset and then counting the number of unique labels predicted by the classifier when 
using samples from a \gls{GAN} as input. 
Both \gls{IS} and \gls{NOM} define diversity in terms of predefined labels, and therefore require knowledge of the ground truth labels and a separate classifier. 

In some discrete domains, diversity is often evaluated in terms of the distribution of unique features. 
For example in \gls{NLP}, a standard metric is n-gram diversity, which is defined as the number of distinct n-grams divided by the total number of n-grams~\citep[e.g.][]{li2016diversity}.
These metrics require an explicit, discrete feature representation.

There are proposed metrics that use similarity scores to define diversity. The most widely used metric of this form is the average pairwise similarity score or the complement, the average dissimilarity.
In text, variants of this metric include \acrshort{PAIRWISE-BLEU}~\citep{shen2019mixture} and \acrshort{D-LEX-SIM}~\citep{fomicheva2020unsupervised}, in which the similarity function is an n-gram overlap metric such as \acrshort{BLEU}~\citep{papineni2002bleu}.
In biology, average dissimilarity is known as IntDiv~\citep{benhenda2017chemgan}, with similarity defined as the Jaccard (Tanimoto) similarity between molecular fingerprints.
Average similarity has some shortcomings, which we highlight in \ref{fig:properties}. 
The figure shows the similarity matrices induced by a shape similarity function and/or a color similarity function. Each of the similarity functions is $1$ when the index of the column and the index of the row have the same shape or color and $0$ otherwise. As shown in \ref{fig:properties}, the average similarity--here measured by IntDiv--becomes less sensitive as diversity increases and does not account for correlations between features. This is not the case for the Vendi Score, which accounts for correlations between features and is able to capture the increased diversity resulting from composing distinct similarity functions. 
Related to the metric we propose here is a similarity-sensitive diversity metric proposed in ecology by~\citet{leinster2012measuring}, and which was introduced in the context of \gls{ML} by~\citet{posada2020gait}.
This metric is based on a notion of entropy defined in terms of a \textit{similarity profile}, a vector whose entries are equal to the expected similarity scores of each element.
Like IntDiv, it does not account for correlations between features.

Some other diversity metrics in the \gls{ML} literature fall outside of these categories.
The Birthday Paradox Test~\citep{Arora2018} aims to estimate the size of the support of a generative model, but requires some manual inspection of samples.
\acrshort{GILBO}~\citep{alemi2018gilbo} is a reference-free metric but is only applicable to latent variable generative models.
~\citet{kviman2022multiple} measure the diversity of ensembles of variational approximations using the \gls{JSD}; this metric is only applicable to sets of probability distributions.
~\citet{mitchell2020diversity} introduce metrics for diversity and inclusion, defining diversity in terms of the representation of socially relevant attributes like gender and race, and using the term \textit{heterogeneity} to refer to variety in arbitrary attributes; in this paper, we use the term diversity to have the same sense as heterogeneity, meaning variety in arbitrary (user-specified) attributes.
In the context of drug exploration,~\citet{xie2022much} propose a metric based on the size of the largest subset of elements such that the similarity between any pair of elements is below some threshold, but this metric requires setting a threshold.
Similarly, in the field of evolutionary computation, \gls{QD} algorithms~\citep{pugh2015confronting}, have assessed diversity by discretizing the feature space into grid of bins and counting the number of covered bins, but this approach requires picking a bin size.

As discussed above, several attempts have been made to measure diversity in \gls{ML}. However, the proposed metrics can be limited in their applicability in that they require a reference dataset or predefined labels, or are domain-specific and applicable to one class of models. The existing metrics that do not have those applicability limitations have shortcomings when it comes to capturing diversity that we have illustrated in \ref{fig:properties}.  

\section{Measuring Diversity with the Vendi Score}
\label{sec:method}

We now define the Vendi Score, state its properties, and study its computational complexity. (We relegate all proofs of lemmas and theorems to the appendix.)

\subsection{Defining the Vendi Score}

To define a diversity metric in \gls{ML} we look to ecology, the field that centers diversity in its work.
In ecology, one main way diversity is defined is as the exponential of the entropy of the distribution of the species under study~\citep{jost2006entropy,leinster2021entropy}. This is a reasonable index for diversity. Consider a population with a uniform distribution over $n$ species, with entropy $\log(n)$. This population has maximal ecological diversity $n$, the same diversity as a population with $n$ members, each belonging to a different species. The ecological diversity decreases as the distribution over the species becomes less uniform, and is minimized and equal to one when all members of the population belong to the same species. 
For a more extensive mathematical discussion of entropy and diversity in the context of biodiversity, we refer readers to~\citet{leinster2021entropy}.
 
How can we extend this way of thinking about diversity to \gls{ML}? One naive approach is to define diversity as the exponential of the Shannon entropy of the probability distribution defined by a machine learning model or dataset. However, this approach is limiting in that it requires a probability distribution for which entropy is tractable, which is not possible in many \gls{ML} settings. We would like to define a diversity metric that only relies on the samples being evaluated for diversity. And we would like for such a metric to achieve its maximum value when all samples are dissimilar and its minimum value when all samples are the same. This implies the need to define a similarity function over the samples. Endowed with such a similarity function, we can define a form of entropy that only relies on the samples to be evaluated for diversity. This leads us to the Vendi Score: 

\begin{definition}[Vendi Score]\label{def:vendi_score}
  Let $x_1, \ldots, x_n \in \gX$ denote a collection of samples, let $k: \gX \times \gX \to \sR$ be a positive semidefinite similarity function, with $k(x, x) = 1$ for all $x$, and let $\mK \in \sR^{n \times n}$ denote the kernel matrix with entry $K_{i, j} = k(x_i, x_j)$. Denote by $\lambda_1, \ldots, \lambda_n$ the eigenvalues of $\mK/n$. The Vendi Score (${\ourmetric}$) is defined as the exponential of the Shannon entropy of the eigenvalues of $\mK/n$:
  \begin{align}\label{eq:vendi-score}
    {\ourmetric}_k(x_1, \ldots, x_n) =  \exp\left( - \sum_{i=1}^n \lambda_i \log \lambda_i \right), 
  \end{align}
where we use the convention $0 \log 0 = 0$.
\end{definition}
To understand the validity of the Vendi Score as a mathematical object, note that the eigenvalues of $\mK/n$ are nonnegative (because $k$ is positive semidefinite) and sum to one (because the diagonal entries of $\mK/n$ are equal to $1/n$). The Shannon entropy is therefore well-defined and the Vendi Score is well-defined.
In this form, the Vendi Score can also be seen as the \emph{effective rank} of the kernel matrix $\mK$. Effective rank was introduced by~\citet{roy2007effective} in the context of signal processing; the effective rank of a matrix is defined as the exponential of the entropy of the normalized singular values.
Effective rank has also been used in machine learning, for example, to evaluate word embeddings~\citep{torregrossa2020correlation} and to study the implicit bias of gradient descent for low-rank solutions~\citep{arora2019implicit}.

The Vendi Score can be expressed directly as a function of the kernel similarity matrix $\mK$:
\begin{lemma}\label{lemma:von_neumann}
Consider the same setting as Definition~\ref{def:vendi_score}. Then 
  \begin{align}
    {\ourmetric}_k(x_1, \ldots, x_n) =  \exp \left (-\tr \left(\frac{\mK}{n} \log \frac{\mK}{n} \right) \right )
    .
  \end{align}
\end{lemma}
The lemma makes explicit the connection of the Vendi Score to quantum statistical mechanics: the Vendi Score is equal to the exponential of the von Neumann entropy associated with $\mK/n$~\citep{bengtsson2017geometry, bach2022information}.
In quantum statistical mechanics, the state of a quantum system is described by a \emph{density matrix}, often denoted $\rho$.
The von Neumann entropy of $\rho$ quantifies the uncertainty in the state of the system~\citep{wilde2013quantum}. 
The normalized similarity matrix $\mK/n$ here plays the role of the density matrix. 

Our formulation of the Vendi Score assumes that $x_1, \ldots, x_n$ were sampled independently, and so $p(x_i) \approx \frac{1}{n}$ for all $i$. This is the usual setting in ML and the setting we study in our experiments.
However, we can generalize the Vendi Score to a setting in which we have an explicit probability distribution over the sample space $\gX$ (see Definition~\ref{def:weighted_vendi_score} in the appendix).

\subsection{Understanding the Vendi Score}
Figure~\ref{fig:properties} illustrates the behavior of the Vendi Score on simple toy datasets in which each element is defined by a shape and a color, and similarity is 
defined to be $1$ if elements share both shape and color, $0.5$ if they share either shape or color, and $0$ otherwise.

First, Figure~\ref{fig:effective_number} illustrates that the Vendi Score is an \textit{effective number}, and can be understood as the effective number of dissimilar elements in a sample.
The value of measuring diversity with effective numbers has been argued in ecology~\citep[e.g.][]{hill1973diversity,patil1982diversity,jost2006entropy} and economics~\citep{adelman1969comment}.
Effective numbers provide a consistent basis for interpreting diversity scores, and make it possible to compare diversity scores using ratios and percentages.
For example, in Figure~\ref{fig:effective_number}, when the number of modes doubles from two to four, the Vendi Score doubles as well. 
If we doubled the number of modes from four to eight, the Vendi Score would double once again. 

Figures~\ref{fig:nonlinear} and~\ref{fig:correlations} illustrate another strength of the Vendi Score, which is that it accounts for correlations between features. 
Given distinct similarity functions $k$ and $k'$, the Vendi Score calculated using the combined similarity function $\frac{1}{2}k(x) + \frac{1}{2}k'(x)$ can be 
greater than the average of the individual Vendi Scores if the two similarity functions describe distinct dimensions of variation. 
Furthermore, the Vendi Score increases when the items in the sample differ in more attributes, and the attributes become less correlated with each other.

The Vendi Score has several desirable properties as a diversity metric. We summarize them in the following theorem.
\begin{theorem}[Properties of the Vendi Score]\label{thm:properties}
  Consider the same definitions in \ref{def:vendi_score} and \ref{def:weighted_vendi_score}.
  \begin{enumerate}
    \item \textbf{Effective number.} If $k(x_i, x_j) = 0$ for all $i \neq j$, then ${\ourmetric}_k(x_1, \ldots, x_n)$ is maximized and equal to $n$. If $k(x_i, x_j) = 1$ for all $i, j$, then ${\ourmetric}_k(x_1, \ldots, x_n)$ is minimized and equal to $1$.
    \item \textbf{Identical elements.} Suppose $k(x_i, x_j) = 1$ for some $i \neq j$.
      Let $\vp'$ denote the probability distribution created by combining $i$ and $j$, i.e. $p'_i = p_i + p_j$ and $p'_j = 0$. Then the Vendi Score is unchanged,
      \begin{itemize}
      	\item[] \centering ${\ourmetric}_k(x_1, \ldots, x_n, \vp) = {\ourmetric}_{k}(x_1, \ldots, x_n, \vp')$.
      \end{itemize}
      \item \textbf{Partitioning.} Suppose $S_1, \ldots, S_m$ are collections of samples such that, for any $i \neq j$, for all $x \in S_i, x' \in S_j$, $k(x, x') = 0$.
      Then the diversity of the combined samples depends only on the diversities of $S_1, \ldots, S_m$ and their relative sizes.
      In particular, if $p_i = |S_i|/\sum_{j}|S_j|$ is the relative size of $S_i$ and $H(p_1, \ldots, p_m)$ denotes the Shannon entropy, then the Vendi Score is the geometric mean,
      \begin{itemize}
      	\item[] \centering ${\ourmetric}_k(S_1, \ldots, S_m) = \exp(H(p_1, \ldots, p_m))\prod_{i=1}^m {\ourmetric}_k(S_i)^{p_i}.$
      \end{itemize}
    \item \textbf{Symmetry.} If $\pi_1, \ldots, \pi_n$ is a permutation of $1, \ldots, n$, then 
     \begin{itemize}
      	\item[] \centering ${\ourmetric}_k(x_1, \ldots, x_n) = {\ourmetric}_k(x_{\pi_1}, \ldots, x_{\pi_n})$.
      \end{itemize}
  \end{enumerate}
\end{theorem}

The effective number property provides a consistent frame of reference for interpreting the Vendi Score: a sample with a Vendi Score of $m$ can be understood to be as diverse as a sample consisting of $m$ completely dissimilar elements.
The identical elements property provides some justification for our use of a sampling approximation: for example, calculating the empirical Vendi Score of a sample of $90$ blue diamonds and $10$ yellow squares is equivalent to calculating the probability-weighted Vendi Score of a sample of one blue spade and one yellow square, with $\vp = (0.9, 0.1)$.
The partitioning property is analogous to the partitioning property of the Shannon entropy and means that if two samples are completely dissimilar we can calculate the diversity of the union of the samples using only the diversity of each sample independently and their relative sizes.
The symmetry property means that the Vendi Score will be the same regardless of how we order the rows and columns in the similarity matrix.

\subsection{Calculating the Vendi Score}
Calculating the Vendi Score for a sample of $n$ elements requires finding the eigenvalues of an $n \times n$ matrix, which has a time complexity of $O(n^3)$.
However, when embeddings of the observations (or feature vectors) are available, which is the case in many \gls{ML} settings and in many of the applications we consider in this paper, one can use similarity functions defined as inner products between the embeddings $\phi(x) \in \sR^d$, with $d \ll n$. That is, we can use the similarity matrix $\mK = \mX^{\top}\mX$, where $\mX \in \sR^{n \times d}$ is the embedding/feature matrix with row $\mX_{i, :} = \phi(x_i)$.
The eigenvalues of $\mK/n$ are the same as the eigenvalues of the covariance matrix $\mX\mX^{\top}/n$, therefore we can calculate the Vendi Score exactly in a time of $O(d^2n + d^3) = O(d^2n)$.
This is the same complexity as existing metrics such as \gls{FID}~\citep{heusel2017gans}, which require calculating the covariance matrix of Inception embeddings.

When embeddings aren't available, the Vendi Score can be approximated using column sampling methods~\citep[i.e. the Nystr\"{o}m method; ][]{williams2000using}.

\paragraph{Sample complexity.}
The Vendi Score is the exponential of the kernel entropy, $H(\mK) = -\tr \left(\frac{\mK}{n} \log \frac{\mK}{n}\right)$.
~\citet{bach2022information} proves that empirical estimator of the kernel entropy has a convergence rate proportional to $1/\sqrt{n}$, where $n$ is the number of samples (Appendix~\ref{app:sample_complexity}).

\subsection{Connections to Other Areas in ML}

Here we remark on the connections between the Vendi Score and other commonly studied objects in \gls{ML} that make use of the eigenvalues of a similarity matrix.

\parhead{Determinantal Point Processes.}
The Vendi Score bears a relationship to \glspl{DPP}, which have been used in machine learning for diverse subset selection~\citep{kulesza2012determinantal}.
A \gls{DPP} is a probability distribution over subsets of a ground set $\mathcal{X}$ parameterized by a positive semidefinite kernel matrix $\mK$.
The likelihood of drawing any subset $X \subseteq \mathcal{X}$ is defined as proportional to $|\mK_{X}|$, the determinant of the similarity matrix restricted to elements in $X$: $p(X) \propto |\mK_{X}| = \prod_i \lambda_i,$
where $\lambda_i$ are the eigenvalues of $\mK_{X}$.
The likelihood function has a geometric interpretation, as the square of the volume spanned by the elements of $X$ in an implicit feature space. 
However, the \gls{DPP} likelihood is not commonly used for evaluating diversity, and has some limitations.
For example, it is always equal to 0 if the sample contains any duplicates, and the geometric meaning is arguably less straightforward to interpret than the Vendi Score, which can be understood in terms of the effective number of dissimilar elements.

\parhead{Spectral Clustering.}
The eigenvalues of the similarity matrix are also related to spectral clustering algorithms~\citep{von2007tutorial}, which use a matrix known as the graph Laplacian, defined $\mL = \mD - \mK$, where $\mK$ is a symmetric, weighted adjacency matrix with non-negative entries, and $\mD$ is a diagonal matrix with $D_{i, i} = \sum_{j} K_{i, j}$.
The eigenvalues of $\mL$ can be used to characterize different properties of the graph---for example, the multiplicity of the eigenvalue 0 is equal to the number of connected components.
As a metric for diversity, the Vendi Score is somewhat more general than the number of connected components: it provides a meaningful measure even for fully connected graphs, and captures within-component diversity.

\section{Experiments}
\label{sec:empirical}

We illustrate the Vendi Score, which we now denote by \acrshort{VS} for the rest of this section,  on synthetic data to illustrate that it captures intuitive notions of diversity, and then apply it to a variety of setting in \gls{ML}.
We used \acrshort{VS} to evaluate the diversity of generative models of molecules, an application where diversity plays an important role in enabling discovery.  
We compare \acrshort{VS} to IntDiv, a function of the average similarity:
\begin{align*}
  \mathrm{IntDiv}(x_1, \ldots, x_n) = 1 - \frac{1}{n^2} \sum_{i, j} k(x_i, x_j).
\end{align*}
We found that \acrshort{VS} identifies some model weaknesses that are not detected by IntDiv.
We also applied \acrshort{VS} to generative models of images, and decoding algorithms of text, where we found it confirms what we know about diversity in those applications. 
We also used \acrshort{VS} to measure mode collapse in \glspl{GAN} and datasets and show that it reveals finer-grained distinctions in diversity than current metrics for measuring mode collapse.
Finally, we used \acrshort{VS} to analyze the diversity of several image, text, and molecule datasets, gaining insights into the diversity profile of those datasets. 
(Implementation details are provided in Appendix~\ref{app:implementation_details}.)

\begin{minipage}{\linewidth}
\centering
\resizebox{\linewidth}{!}{\includegraphics{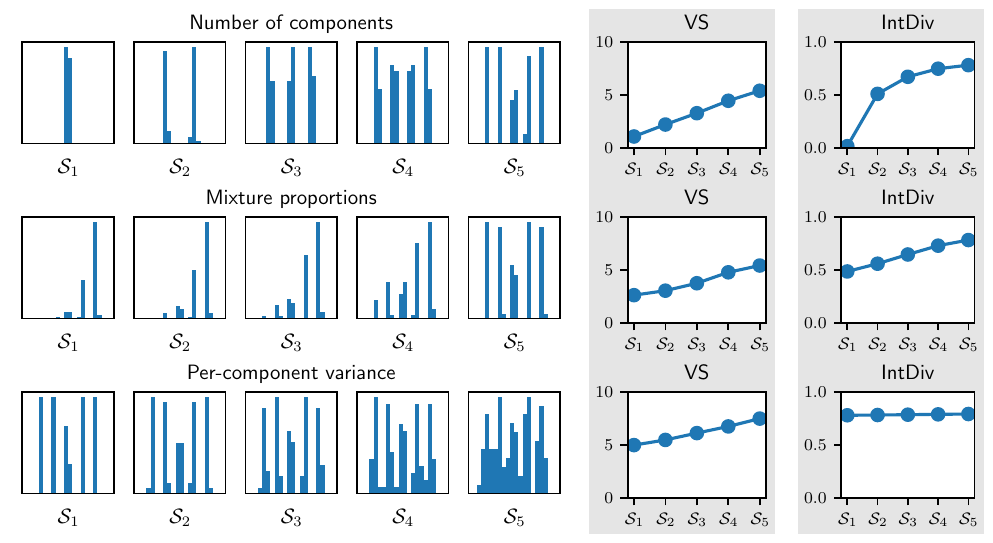}
}
\captionof{figure}{
\label{fig:synthetic_example}
${\ourmetric}$ increases proportionally with diversity in three sets of synthetic datasets. 
In each row, we sample datasets from univariate mixture-of-normal distributions,
varying either the number of components, the mixture proportions, or the per-component variance. 
The datasets are depicted in the left, as histograms, and the diversity scores are plotted on the right.
}
\end{minipage}

\subsection{Synthetic experiments}
\label{sec:synthetic}

To illustrate the behavior of the Vendi Score, we calculate the diversity of simple datasets drawn from a mixture of univariate normal distributions, varying either the number of components, the mixture proportions, or the per-component variance.
We measure similarity using the RBF kernel: $k(x, x') = \mathrm{exp}(\lVert x - x '\rVert^2 / 2\sigma^2)$.
The results are illustrated in Figure~\ref{fig:synthetic_example}.
${\ourmetric}$ behaves consistently and intuitively in all three settings:
in each case, ${\ourmetric}$ can be interpreted as the effective number of modes, ranging between one and five in the first two rows and increasing from five to seven in the third row as we increase within-mode variance.
On the other hand, the behavior of IntDiv is different in each settings:
for example, IntDiv is relatively insensitive to within-mode variance, and additional modes bring diminishing returns.

In Appendix~\ref{app:mode_dropping_in_datasets}, we also validate that \acrshort{VS} captures mode dropping in a simulated setting, using image and text classification datasets, where we have information about the ground truth class distribution.
In both cases, \acrshort{VS} has a stronger correlation with the true number of modes compared to IntDiv.

\subsection{Evaluating molecular generative models for diversity}
\label{sec:molecule_models}
\begin{figure}[t]
  \centering
  \includegraphics[width=\textwidth]{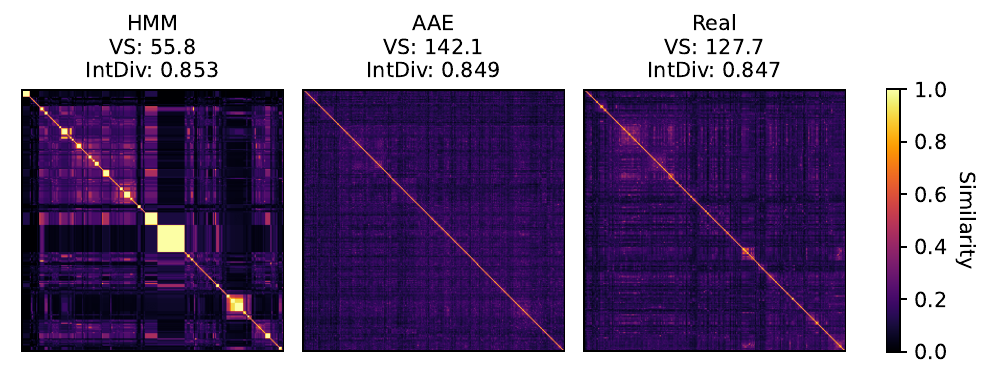}
  \caption{The kernel matrices for $250$ molecules sampled 
    from the \acrshort{HMM}, \acrshort{AAE}, and the original dataset, sorted
    lexicographically by \acrshort{SMILES} string representation.
    The samples have similar IntDiv scores,
    but the \acrshort{HMM} samples score much lower on \acrshort{VS}.
    The figure shows that the \acrshort{HMM} generates a number of exact duplicates. 
    \acrshort{VS} is able to capture the \acrshort{HMM}'s lack of diversity while IntDiv cannot.}
\label{fig:molecule_kernels}
\end{figure}

Next, we evaluate the diversity of samples from generative models of molecules. For generative models to be 
useful for the discovery of novel molecules, they ought to be diverse. The standard diversity metric in this setting is IntDiv.
We evaluate samples from generative models provided in the \acrshort{MOSES} benchmark~\citep{polykovskiy2020moses}, using the first $2500$ valid molecules in each sample.
Following prior work, our similarity function is the Morgan fingerprint similarity (radius 2), implemented in RDKit.\footnote{RDKit: Open-source Cheminformatics. https://www.rdkit.org.}
In Figure~\ref{fig:molecule_kernels}, we highlight an instance where \acrshort{VS} and IntDiv disagree: IntDiv ranks the \acrshort{HMM} among the most diverse models, while \acrshort{VS} ranks it as the least diverse (the complete results are in Appendix Table~\ref{tab:molecule_metrics}).
The \acrshort{HMM} has a high IntDiv score because, on average, the \acrshort{HMM} molecules have low pairwise similarity scores,
but there are a number of clusters of identical or nearly identical molecules.

\subsection{Assessing mode collapse in GANs}
Mode collapse is a failure mode of \glspl{GAN} that has received a lot of attention from the \gls{ML} community~\citep{Metz2017, dieng2019prescribed}. 
The main metric for measuring mode collapse, called \acrlong{NOM}(\gls{NOM}), can only be used to assess mode collapse for \glspl{GAN} trained on a labelled dataset. 
\gls{NOM} is computed by training a classifier on the labeled training data and counting the number of unique classes that are predicted by the trained classifier for the generated samples. 
In Table~\ref{tab:mode_dropping}, we evaluate two models that were trained on the Stacked\acrshort{MNIST} dataset, a standard setting for evaluating mode collapse in \glspl{GAN}.
Stacked\acrshort{MNIST} is created by stacking three \acrshort{MNIST} images along the color channel, creating $1000$ classes corresponding to $1000$ number of modes.

\begin{table*}[ht]
  \small
\centering
\begin{tabular}{l c c c c}
  \toprule
Model & \acrshort{NOM} & Mode Div. & \acrshort{VS}\\
\midrule
Self-cond. \gls{GAN} & 1000 & 921.0 & 746.7 \\
Pres\gls{GAN} & 1000 & 948.7 & 866.6 \\
\midrule
Original & 1000 & 950.8 & 943.7 \\
\bottomrule
\end{tabular}
\caption{
\label{tab:mode_dropping}
\acrshort{VS} captures a more fine-grained notion of diversity than \acrlong{NOM}(\acrshort{NOM}). 
Although Pres\gls{GAN} and Self-cond.\gls{GAN} both capture all the $1000$ modes, \acrshort{VS} reveals that Pres\gls{GAN} is more diverse than Self-cond.\gls{GAN} 
and that they both are less diverse than the original dataset. 
}
\end{table*}

In prior work, mode collapse is evaluated by training an \acrshort{MNIST} classifier
and counting the number of unique classes that are predicted for the generated samples.
We adapt this approach and 
we calculate \acrshort{VS} using the probability product kernel~\citep{jebara2004probability}: $k(x, x') = \sum_{y} p(y \mid x)^{\frac{1}{2}} p(y \mid x')^{\frac{1}{2}}$, where the class likelihoods are given by the classifier.
We compare Pres\acrshort{GAN}~\citep{dieng2019prescribed} and Self-conditioned \gls{GAN}~\citep{liu2020selfconditioned}, two \glspl{GAN} that are known to capture all the modes.
Table~\ref{tab:mode_dropping} shows that Pres\acrshort{GAN} and Self-conditioned \gls{GAN} have the same diversity according to number of modes, they capture all $1000$ modes. 
However, \acrshort{VS} reveals a more fine-grained notion of diversity, indicating that Pres\acrshort{GAN} is more diverse than Self-conditioned \gls{GAN} and that both are less diverse than the original dataset. 
One possibility is that \acrshort{VS} is capturing imbalances in the mode distribution.
To see whether this is the case, we also calculate what we call \emph{Mode Diversity}, the exponential entropy of the predicted mode distribution: $\exp H(\hat{p}(y))$, where $\hat{p}(y) = \frac{1}{n} \sum_{i=1}^n p(y \mid x_i)$. The generative models score lower on \acrshort{VS} than Mode Diversity, indicating that low scores cannot be entirely attributed to imbalances in the mode distribution. 
Therefore \acrshort{VS} captures more aspects of diversity, even when we are using the same representations as existing methods. 

\subsection{Evaluating image generative models for diversity}
\label{sec:images}

\begin{table*}[!t]
  \small
  \setlength{\tabcolsep}{4pt}
    \begin{minipage}{0.49\linewidth}
      \centering
        \begin{tabular}{l r r r r r}
          \toprule
          Model & \acrshort{IS}$_\uparrow$ & \acrshort{FID}$_\downarrow$ & Prec$_\uparrow$ & Rec$_\uparrow$ & \acrshort{VS}$_\uparrow$ \\
          \\
          \multicolumn{6}{l}{\textbf{\acrshort{CIFAR}-10}} \\
          \midrule
          Original &  &  &  &  & 19.50 \\
          \acrshort{VDVAE} & 5.82 & 40.05 & 0.63 & 0.35 & 12.87 \\
          DenseFlow & 6.01 & 34.54 & 0.62 & 0.38 & 13.55 \\
          \acrshort{IDDPM} & 9.24 & 4.39 & 0.66 & 0.60 & 16.86 \\
          \\
           \multicolumn{5}{l}{\textbf{LSUN Cat 256$\times$256}} \\
          \midrule
          Original &  &  &  &  & 15.12 \\
          Style\acrshort{GAN}2 & 4.84 & 7.25 & 0.58 & 0.43 & 13.55 \\
          \acrshort{ADM} & 5.19 & 5.57 & 0.63 & 0.52 & 13.09 \\
          \acrshort{RQ-VT} & 5.76 & 10.69 & 0.53 & 0.48 & 14.91 \\
        \end{tabular}
\end{minipage}\hfill
    \begin{minipage}{.49\linewidth}
      \centering
        \begin{tabular}{ r r r r r}
          \toprule
         \acrshort{IS}$_\uparrow$ &\acrshort{FID}$_\downarrow$ & Prec$_\uparrow$ & Rec$_\uparrow$ & \acrshort{VS}$_\uparrow$ \\
          \\
          \multicolumn{5}{l}{\textbf{ImageNet 64$\times$64}} \\
          \midrule
            &  &  &  & 43.93 \\
           9.68 & 57.57 & 0.47 & 0.37 & 18.04 \\
          5.62 & 102.90 & 0.36 & 0.17 & 12.71 \\
         15.59 & 19.24 & 0.59 & 0.58 & 24.28 \\
          \\
           \multicolumn{5}{l}{\textbf{LSUN Bedroom 256$\times$256}} \\
          \midrule
            &  &  &  & 8.99 \\
          2.55 & 2.35 & 0.59 & 0.48 & 8.76 \\
           2.38 & 1.90 & 0.66 & 0.51 & 7.97 \\
           2.56 & 3.16 & 0.60 & 0.50 & 8.48 \\
        \end{tabular}
\end{minipage} 
\caption{
\label{tab:comparing_image_models}
\acrshort{VS} generally agrees with the existing metrics. On low-resolution datasets (top left and top right) the diffusion model performs better on all of the metrics. 
On the \acrshort{LSUN} datasets (bottom left and bottom right), the diffusion model gets the highest quality scores as measured by \acrshort{IS}, but scores lower on \acrshort{VS}. 
No model matches the diversity score of the original dataset they were trained on.
}
\end{table*}

We now evaluate several recent models for unconditional image generation, comparing the diversity scores with standard evaluation metrics, \gls{IS}~\citep{Salimans2016}, \gls{FID}~\citep{heusel2017gans}, Precision~\citep{sajjadi2018assessing}, and Recall~\citep{sajjadi2018assessing}.
The models we evaluate represent popular classes of generative models, including a variational autoencoder~\citep[\acrshort{VDVAE}; ][]{child2020very}, a flow model~\citep[DenseFlow; ][]{grcic2021densely}, diffusion models (\acrshort{IDDPM}, \citealp{nichol2021improved}; \acrshort{ADM}~\citealp{dhariwal2021diffusion}), \gls{GAN}-based models~\citep{karras2019style,karras2020analyzing}, and an auto-regressive model~\citep[\acrshort{RQ-VT}; ][]{lee2022autoregressive}.
The models are trained on \acrshort{CIFAR}-10~\citep{Krizhevsky2009}, ImageNet~\citep{russakovsky2015imagenet}, or two categories from the \acrshort{LSUN} dataset~\citep{yu2015lsun}.
We either select models that provide precomputed samples, or download publicly available model checkpoints and sample new images using the default hyperparameters. (More details are in Appendix~\ref{app:implementation_details}.)

The standard metrics in this setting use a pre-trained Inception ImageNet classifier to map images to real vectors.
Therefore, we calculate \acrshort{VS} using the cosine similarity between Inception embeddings, using the same 2048-dimensional representations used for evaluating \acrshort{FID} and Precision/Recall.
As a result, the highest possible similarity score is 2048.
The baseline metrics are reference-based, with the exception of \acrshort{IS}. 
\acrshort{FID} and \acrshort{IS} capture diversity implicitly. Recall was introduced to capture diversity explicitly, with diversity defined as coverage of the reference distribution.

The results of this comparison are in Table~\ref{tab:comparing_image_models}.
On the lower resolution datasets (top left and top right), \acrshort{VS} generally agrees with the existing metrics.
On those datasets the diffusion model performs better on all of the metrics.
  On the \acrshort{LSUN} datasets (bottom left and bottom right), the diffusion model gets the highest quality scores as measured by precision and recall, but scores lower on \acrshort{VS}.
  In these cases, \acrshort{VS} can be interpreted as complementing the existing metrics.
  For example, on \acrshort{LSUN} Cat, the ADM model achieves a precision score of 0.63 and recall of 0.52, implying that 63\% of generated images look like reference images, and that the generated images cover 52\% of the reference distribution; however, the low \acrshort{VS} suggests that the remaining images have low internal diversity---for example, the model may generate many near-duplicates.
No model matches the diversity score of the original dataset they were trained on.
In addition to comparing the diversity of the models, we can also compare the diversity scores between datasets:
as a function of Inception similarity, the most diverse dataset is ImageNet 64$\times$64, followed by \acrshort{CIFAR}-10, 
followed by \acrshort{LSUN} Cat, and then \acrshort{LSUN} Bedroom.
 Cat (all cats, but coming in different species), followed by \acrshort{LSUN} Bedrooms.

\acrshort{VS} should be understood as the diversity with respect to a specific similarity function, in this case, the Inception ImageNet similarity.
We illustrate this point in in the appendix (Figure~\ref{fig:lsun_cat_comparing_kernels}) by comparing the top eigenvalues of the kernel matrices corresponding to the 
cosine similarity between Inception embeddings and pixel vectors.
Inception similarity captures a form of semantic similarity, with components corresponding to particular cat breeds, while the pixel kernel provides a 
simple form of visual similarity, with components corresponding to broad differences in lightness, darkness, and color.

\subsection{Evaluating decoding algorithms for text for diversity}
\label{sec:text}
\begin{table*}[ht]
\centering
\begin{tabular}{l c c c}
  \toprule
Source & BLEU & N-gram diversity ($\uparrow$) & $D^K_1$ ($\uparrow$)\\
\midrule
Human &  & 0.82 & 3.12\\
  \midrule
Beam Search & 0.27 & 0.42 & 2.44\\
DBS $\gamma=0.2$ & 0.25 & 0.49 & 2.49\\
DBS $\gamma=0.5$ & 0.22 & 0.63 & 2.87\\
DBS $\gamma=0.5$ & 0.21 & 0.68 & 2.95\\
\bottomrule
\end{tabular}
\caption{
\label{tab:caption_diversity}
Quality and diversity scores for an image captioning model using different decoding algorithms. \textbf{BS:} Beam search. \textbf{DBS:} Diverse beam search~\citep{vijayakumar2018diverse}, varying the diversity penalty $\gamma$.
BLEU measures n-gram overlap with the human-written reference captions, a proxy for quality.
${\ourmetric}$ is calculated using a BLEU score kernel.
Using Diverse Beam Search with higher diversity penalties leads to higher diversity scores, according to both metrics, but a lower quality score.
The underlying model is a GPT-2-based model trained on MS COCO.
}
\end{table*}

We evaluate diversity on the \acrshort{MS COCO} image-captioning dataset~\citep{lin2014microsoft}, following prior work on diverse text generation~\citep{vijayakumar2018diverse}.
In this setting, the subjects of evaluation are diverse decoding algorithms rather than parametric models.
Given a fixed conditional model of text $p(x \mid c)$, where $c$ is some conditioning context, the aim is to identify a ``Diverse N-Bet List'', a list of sentences that have high likelihood but are mutually distinct.
The baseline metric we compare to is n-gram diversity~\citep{li2016diversity},
which is the proportion of unique n-grams divided by the total number of n-grams. We define similarity using the n-gram overlap kernel:
for a given $n$, the n-gram kernel $k_n$ is the cosine similarity between bag-of-n-gram feature vectors.
We use the average of $k_1, \ldots, k_4$.
This ensures that \acrshort{VS} and n-gram diversity are calculated using the same feature representation.
Each image in the validation split has five captions written by different human annotators, and we compare these with captions generated by a 
publicly available captioning model trained on this dataset~\footnote{https://huggingface.co/ydshieh/vit-gpt2-coco-en-ckpts}.
For each image, we generate five captions using either beam search or \gls{DBS}~\citep{vijayakumar2018diverse}.
\gls{DBS} takes a parameter, $\gamma$, called the diversity penalty, and we vary this between $0.2$, $0.6$, and $0.8$.

Table~\ref{tab:caption_diversity} shows that all diversity metrics increase as expected, ranking beam search the lowest, the human 
captions the highest, and \gls{DBS} in between, increasing with the diversity penalty.
The human diversity score of $4.88$ can be interpreted as meaning that, on average, all five human-written captions 
are almost completely dissimilar from each other, while beam search effectively returns only three distinct responses for every five that it generates.

\subsection{Diagnosing datasets for diversity}
\label{sec:diagnosing_datasets}

\begin{figure*}[t]
     \centering
      \includegraphics[width=\linewidth]{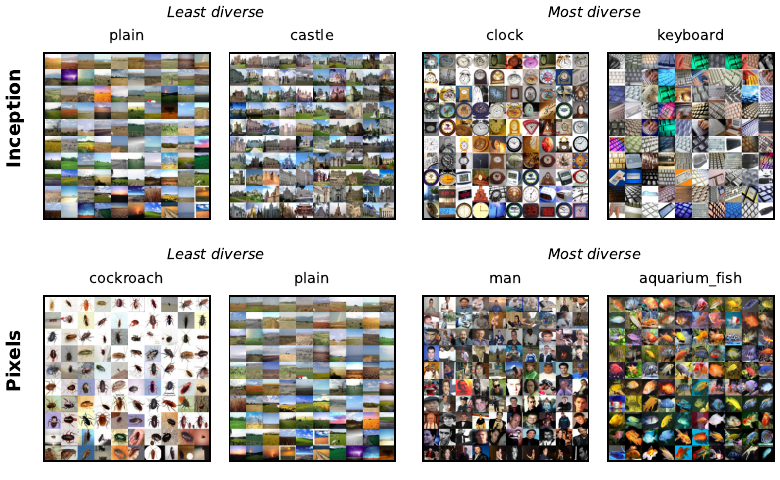}
\caption{
\label{fig:image_categories}
The categories in \acrshort{CIFAR}-100 with the lowest and highest \acrshort{VS}, defining similarity as the cosine similarity between either Inception embeddings or pixel vectors.
We show $100$ examples from each category, in decreasing order of average similarity, with the image at the top left having the highest average similarity scores according to the corresponding kernel.
}
\end{figure*}
In Figure~\ref{fig:image_categories},
we calculate \acrshort{VS} for samples from different categories in \acrshort{CIFAR-100}, 
using the cosine similarity between either Inception embeddings or pixel vectors.
The pixel diversity is highest for categories like ``aquarium fish'', which vary in color, brightness, and orientation, and lowest for categories like ``cockroach'' in which images have similar regions of high pixel intensity (like white backgrounds).
The Inception diversity is less straightforward to interpret, but might correspond to some form of semantic diversity---for example, the Inception diversity might be lower for classes like ``castle,'' that correspond to distinct ImageNet categories, and higher for categories like ``clock'' and ``keyboard'' that are more difficult to classify.
In Appendix~\ref{app:diagnosing_datasets}, we show additional examples from text, molecules, and other image datasets.

\section{Limitations}
\label{sec:limitations}
Here, we discuss several things to consider when interpreting \acrshort{VS} scores.
First, \acrshort{VS} is a reference-free metric, meaning that it measures the internal diversity of a set and not how it relates to a reference distribution. While this makes \acrshort{VS} useful in settings where there is no reference distribution, it also means that it is possible to get a high diversity score by, for example, sampling random noise. This is also true of other reference-free metrics, like IntDiv and n-gram diversity. Therefore, \acrshort{VS} should be used alongside a quality metric.
Second, like other similarity-based metrics, \acrshort{VS} is dependent on the choice of similarity function. If the similarity function is too sensitive, all sets will appear very diverse, while if it is not sensitive enough, all sets will have low diversity. Additionally, the wrong choice of similarity function can introduce biases that lead to skewed diversity scores. Therefore, care should be taken when choosing a similarity function to ensure that it is appropriate for the specific application.
Finally, the computational cost of calculating \acrshort{VS} can be high when the similarity function is not associated with low-dimensional embeddings.

\section{Discussion}
\label{sec:discussion}
We introduced the Vendi Score, a metric for evaluating diversity in \gls{ML}. 
The Vendi Score is defined as a function of the pairwise similarity scores between elements of a sample and can be interpreted as the effective number of unique elements in the sample. The Vendi Score is interpretable, general, and applicable to any domain where similarity can be defined. It is unsupervised, in that it does not require labels or a reference probability distribution or dataset. Importantly, the Vendi Score allows its user to specify the form of diversity they want to measure via the similarity function. We showed the Vendi Score can be computed efficiently exactly and showcased its usefulness in several \gls{ML} applications, different datasets, and different domains. In future work, we will leverage the Vendi Score to improve data augmentation, an important \gls{ML} approach in settings with limited data.
 
\subsection*{Acknowledgements}
Adji Bousso Dieng is supported by the National Science Foundation, Office of Advanced Cyberinfrastructure (OAC): \#2118201.
We thank Sadhika Malladi for pointing us to the effective rank.
Adji Bousso Dieng would like to dedicate this paper to her PhD advisors, David Blei and John Paisley.

\clearpage
\bibliographystyle{apa}
\bibliography{main}

\clearpage

\section{Proofs}
\label{app:properties}

\subsection{Probability-weighted Vendi Score}
\begin{definition}[Probability-Weighted Vendi Score]\label{def:weighted_vendi_score}
  Let $\vp \in \Delta_n$ denote a probability distribution on a discrete space $\gX = \{x_1, \ldots, x_n\}$, where $\Delta_n$ denotes the $(n-1)$-dimensional simplex, let $k: \gX \times \gX \to \sR$ be a positive semidefinite similarity function, with $k(x, x) = 1$ for all $x$, and let $\mK \in \sR^{n \times n}$ denote the kernel matrix with $K_{i, j} = k(x_i, x_j)$.
Let $\tilde{\mK}_{\vp} = \mathrm{diag}(\sqrt{\vp})\mK\mathrm{diag}(\sqrt{\vp})$
denote the probability-weighted kernel matrix. Let $\lambda_1, \ldots, \lambda_n$ denote the eigenvalues of $\tilde{\mK}_{\vp}$. 
  The Vendi Score (${\ourmetric}$) is defined as the exponential of the Shannon entropy of the eigenvalues of $\tilde{\mK}_{\vp}$:
  \begin{align}\label{eq:def-general}
    {\ourmetric}_k(x_1, \ldots, x_n, \vp) =  \exp\left( - \sum_{i=1}^S \lambda_i \log \lambda_i \right).
  \end{align}
\end{definition}
When all elements in the sample are completely dissimilar, the probability-weighted Vendi Score defined in~\ref{def:weighted_vendi_score} reduces to the exponential of the Shannon entropy of the weighting distribution:
\begin{lemma}\label{lemma:dissimilar}
Let $\vp \in \Delta_n$ be a probability distribution over $x_1, \ldots, x_n$ and suppose $k(x_i, x_j) = 0$ for all $i \neq j$. Then ${\ourmetric}_k(x_1, \ldots, x_n, \vp) = \exp H(\vp)$, the exponential of the Shannon entropy of $\vp$.
\end{lemma}

\subsection{Proof of~\ref{lemma:von_neumann}}
\begin{lemma*}
Consider the same setting as Definition~\ref{def:vendi_score}. Then 
  \begin{align}
    {\ourmetric}_k(x_1, \ldots, x_n) =  \exp \left (-\tr \left(\frac{\mK}{n} \log \frac{\mK}{n} \right) \right )
    .
  \end{align}
\end{lemma*}

\begin{proof}
  For any square matrix $\mX \in \sR^{n \times n}$, if $\mX$ has an eigendecomposition $\mX = \mU \mLambda \mU^{-1}$, then $\log \mX = \mU \left ( \log \mLambda \right) \mU^{-1}$, where $\log \mLambda = \mathrm{diag}(\log \lambda_1, \ldots, \log \lambda_n)$ is a diagonal matrix whose diagonal entries are the logarithms of the eigenvalues of $\mX$. 
  Also, $\tr (\mX) = \tr \left (\mU \mLambda  \mU^{-1} \right) = \tr (\mLambda),$ because the trace is similarity-invariant.
  $\mK/n$ is diagonalizable because it is positive semidefinite,
  so let $\mK/n = \mU \mLambda  \mU^{-1}$ denote the eigendecomposition. Then
  \begin{align*}
    \tr (\mK/n \log \mK/n) &= \tr \left ( \mU \mLambda  \mU^{-1} \log \left (\mU \mLambda  \mU^{-1}  \right) \right)\\
    &= \tr \left ( \mU \mLambda  \mU^{-1} \mU \left (\log \mLambda \right) \mU^{-1} \right)\\
    &= \tr \left ( \mLambda  \log \mLambda \right)\\
    &= \sum_{i=1}^n \lambda_i \log \lambda_i.
  \end{align*}
  Therefore $$
  {\ourmetric}_k(x_1, \ldots, x_n) =  \exp \left (-\sum_{i=1}^n \lambda_i \log \lambda_i \right) = \exp \left (-\tr \left(\frac{\mK}{n} \log \frac{\mK}{n} \right) \right ).
  $$
\end{proof}

\subsection{Proof of~\ref{lemma:dissimilar}}
\begin{lemma*}
Let $\vp \in \Delta_n$ be a probability distribution over $x_1, \ldots, x_n$ and suppose $k(x_i, x_j) = 0$ for all $i \neq j$. Then ${\ourmetric}_k(x_1, \ldots, x_n, \vp) = \exp H(\vp)$, the exponential of the Shannon entropy of $\vp$.
\end{lemma*}
\begin{proof}
  If all element in $\vp$ are completely dissimilar, then $\tilde{\mK}_{\vp}$ is a diagonal matrix, and the eigenvalues $\lambda_1, \ldots, \lambda_S$ are the diagonal entries, which are the entries of $\vp$.
  So the von Neumann entropy of $\tilde{\mK}_{\vp}$ is identical to the Shannon entropy of $\vp$, and the exponential is the Vendi Score.
\end{proof}

\subsection{Proof of Theorem~\ref{thm:properties}}

\begin{proof}
  (a) Effective number: If $\vp$ is the uniform distribution over $N$ completely dissimilar elements, then $\tilde{\mK}_{\vp}$ is a diagonal matrix with each diagonal entry equal to $1/N$. The eigenvalues of a diagonal matrix are the diagonal entries, so ${\ourmetric}_K(\vp) = \exp H(1/N, \ldots, 1/N) = \exp \log N = N$. On the other hand, if all elements are completely similar to each other, then $\tilde{\mK}_{\vp}$ has rank one and so the Vendi Score is equal to one.

(b) Identical elements:  The eigenvalues of $\tilde{\mK}_{\vp}$ are the same as the eigenvalues of the covariance matrix of the corresponding feature space:
\begin{align*}
\tilde{\mSigma}_{\vp} &= \sum_{i=1}^N p(x_i) \phi(x_i) \phi(x_i)^{\top}.\\
\end{align*}
Suppose elements $i$ and $j$ are identical, and let $\vp'$ denote the probability distribution created by combining $i$ and $j$, i.e. $p'_i = p_i + p_j$ and $p'_j = 0$. Clearly, $\tilde{\mSigma}_{\vp} = \tilde{\mSigma}_{\vp'}$, and so ${\ourmetric}_k(x_1, \ldots, x_n, \vp) = {\ourmetric}_{k}(x_1, \ldots, x_n, \vp')$.

(c) Partitioning: Suppose $N$ samples are partitioned into $M$ groups $\gS_1, \ldots, \gS_M$ such that, for any $i \neq j$, for all $x \in S_i, x' \in S_j$, $k(x, x') = 0$.
Let $p_i = |S_i|/\sum_j |S_j|$ denote the relative size of group $i$, and let $\mK$ denote kernel matrix of $\cup_i S_i$, sorted in order of group index, and let $\mK_{S_i}$ denote the restriction of $\mK$ to elements in $S_i$.
Then $\mK/N$ is a block diagonal matrix, with each block $i$ equal to $p_i\mK_{S_i}$.
The eigenvalues of a block diagonal matrix are the combined eigenvalues of each block, and the partitioning property then follows from the partitioning property of the Shannon entropy.

(e) Symmetry: The eigenvalues of a matrix are unchanged by orthonormal transformation, and the Shannon entropy is symmetric in its arguments, so the Vendi Score is symmetric.
\end{proof}

\subsection{Sample Complexity}
\label{app:sample_complexity}
The Vendi Score is the exponential of the kernel entropy, $H(\mK) = -\tr \left(\frac{\mK}{n} \log \frac{\mK}{n}\right)$.
~\citet{bach2022information} proves that empirical estimator of the kernel entropy, $\hat{H}$, has a convergence rate proportional to $1/\sqrt{n}$, where $n$ is the number of samples. Additionally, by Jensen's inequality, the $\mathbb{E}[\hat{H}]$ is no greater than $H$. Therefore:
\begin{align*}
\exp(H) - \exp(\hat{H}) &\leq \exp(H) - \exp\left(H - \frac{1}{\sqrt{n}}\right) \\
                        &= \exp(H) - \exp(H)/\exp\left(\frac{1}{\sqrt{n}}\right) \\
                        &= \exp(H)\left(1 - \frac{1}{\sqrt{n}}\right). \\
\end{align*}
The empirical estimator of the Vendi Score therefore also has a convergence rate proportional to $1/\sqrt{n}$, with a constant term depending on the true entropy.

\section{Implementation Details}
\label{app:implementation_details}

\subsection{Images}
\label{app:image_implementation_details}

\paragraph{Stacked MNIST}
We train GANs on Stacked MNIST using the publicly available code for PresGANs~\footnote{https://github.com/adjidieng/PresGANs} and self-conditioned GANs~\footnote{https://github.com/stevliu/self-conditioned-gan}.
The models share the same DCGAN~\citep{Radford2015} architecture and are trained on the same dataset of 60,000 Stacked MNIST images, rescaled to 32$\times$32 pixels, and other hyperparameters are set according to the descriptions in the papers.
The models are trained for 50 epochs and the diversity scores are evaluated every five epochs by taking 10,000 samples.
For both models, we report the scores from the epoch corresponding to the highest ${\ourmetric}$ score.
As in prior work~\citep{Metz2017}, we classify Stacked MNIST digits by applying a pretrained MNIST classifier to each color channel independently.
The 1000-dimensional Stacked MNIST probability vector is then the tensor product of the three 10-dimensional probability vectors predicted for the three channels.

\paragraph{Obtaining Image Samples}
In Section~\ref{sec:images}, we calculate the diversity scores of
several recent generative models of images. We select models that
represent a range of families of generative models and and provide
publicly available samples or model checkpoints for common image
datasets. On the low-resolution datasets, we generate 50,000 samples
from each model using the official code for
VDVAE,\footnote{https://github.com/openai/vdvae/}
DenseFlow,\footnote{https://github.com/matejgrcic/DenseFlow}, and
IDDPM,\footnote{https://github.com/openai/improved-diffusion}, each of
which provides a checkpoint for unconditional image generation models
on CIFAR-10 and ImageNet-64. For IDDPM, we sample using
DDIM~\citep{song2021denoising} for 250 steps, and otherwise use the
default sampling parameters. For the higher-resolution datasets, we
use the 50,000 precomputed samples provided
by~\citet{dhariwal2021diffusion}\footnote{https://github.com/openai/guided-diffusion}
for ADM and StyleGAN models.
We obtain 50,000 samples from the RQ-VAE/Transformer model using the code and checkpoints provided by the authors,\footnote{https://github.com/kakaobrain/rq-vae-transformer} with the default sampling parameters.

\paragraph{Calculating Image Metrics}
In Table~\ref{tab:comparing_image_models}, we calculate standard image quality and diversity metrics, which are based on Inception embeddings.
These Inception-based metrics are sensitive to a number of implementation details~\citep{parmar2022aliased} and in general cannot be compared directly between papers. 
For a consistent comparison, we calculate all scores using the evaluation code provided by~\citet{dhariwal2021diffusion}. We also calculated FID and Precision/Recall using the provided reference images and statistics, with the exception of CIFAR-10, for which we use the training set as the reference.
(The diversity scores of the \textit{Original} datasets in Table~\ref{tab:comparing_image_models} are calculated using these reference images.)
As a result, the numbers in this table may not be directly comparable to results reported in prior work.

\subsection{Text}
\label{app:text_implementation_details}

\paragraph{Obtaining Image Captions}
In Section~\ref{sec:text}, we sample image captions from a pretrained image-captioning model,\footnote{https://huggingface.co/ydshieh/vit-gpt2-coco-en-ckpts} which is publicly available in Hugging Face~\citep{wolf2019huggingface}, and we use the Hugging Face implementation of beam search and diverse beam search.
For beam search we use a beam size of 5. For diverse beam search, we use a beam size of 10, a beam group size of 10, and set the number of return sequences to 5.

\paragraph{Calculating Text Metrics}
The text metrics we use are calculated in terms of word n-grams, and therefore depend on how sentences are tokenized into words.
We calculate all text metrics using the pre-trained wordpiece tokenizer used by the captioning models.
We use the implementation of the BLEU score in NLTK~\citep{bird2006nltk}.

\section{Additional Results}

\subsection{Assessing Mode Dropping in Datasets}
\label{app:mode_dropping_in_datasets}

\begin{figure}[h]
  \centering
\includegraphics[width=0.49\textwidth]{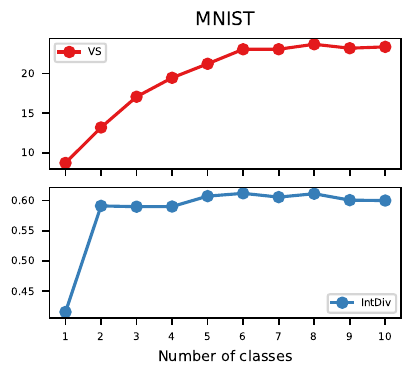}\includegraphics[width=0.49\textwidth]{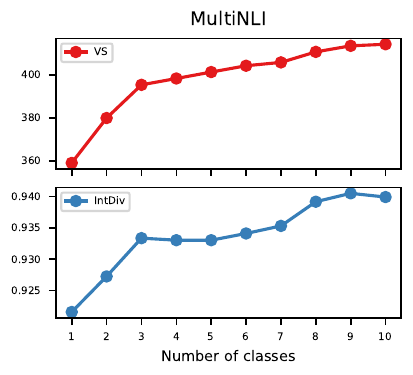}\caption{Detecting mode dropping in image and text datasets. 
    We evaluate \acrshort{VS} and IntDiv on datasets containing $500$ examples drawn uniformly from between one and ten classes: 
    digits in \acrshort{MNIST} and sentences genres in Multi\acrshort{NLI}.
Compared to IntDiv, \acrshort{VS} increases more consistently with the number of classes.
  } \label{fig:mode_dropping}
\end{figure}

In Figure~\ref{fig:mode_dropping}, we examine whether ${\ourmetric}$ captures mode dropping in a controlled setting, where we have information about the ground truth class distribution.
We simulate mode dropping by sampling equal-sized subsets of two classification datasets, with each subset $\gS_i$ containing examples sampled uniformly from the first $i$ categories. 
We perform this experiment one image dataset (\acrshort{MNIST}) and one text dataset (\acrshort{MultiNLI};~\citealp{williams2018broad}), using simple similarity functions.
We compare \acrshort{VS} to the Internal Diversity (IntDiv), defined as above.

\acrshort{MNIST} consists of 28$\times$28-pixel images of hand-written digits, divided into ten classes.
The similarity score we use is the cosine similarity between pixel vectors: $k(x, x') = \langle x, x' \rangle / \|x\| \|x'\|$,
where $x, x'$ are $28^2$-dimensional vectors with entries specifying pixel intensities between 0 and 1.
\acrshort{MultiNLI} is a multi-genre sentence-pair classification dataset.
We use the premise sentences from the validation split (mismatched), which are drawn from one of ten genres.
We define similarity using the n-gram overlap kernel:
for a given $n$, the n-gram kernel $k_n$ can be expressed as the cosine similarity between feature vectors $\phi^n(x)$, where $\phi_i^n(x)$ is equal to the number of times n-gram $i$ appears in $x$.
We use the average of $k(x, x') = \frac{1}{4}\sum_{n=1}^4 k_n(x, x')$.

The results (Figure~\ref{fig:mode_dropping}) show that \acrshort{VS} generally increases with the number of classes, even using these simple similarity scores.
In \acrshort{MNIST} (left), \acrshort{VS} increases roughly linearly for the first six digits (0-5) and then fluctuates. This could occur if the new modes are similar to the other modes in the sample, or have low internal diversity.
In \acrshort{MultiNLI} (right), \acrshort{VS} increases monotonically with the number of genres represented in the sample.
In both cases, \acrshort{VS} has a stronger correlation with the number of modes compared to IntDiv.

\subsection{Evaluating molecular generative models for diversity}
\label{app:molecule_models}
\begin{table}[t]
\centering
\begin{tabular}{lcr}
  \toprule
  Model & IntDiv & \acrshort{VS}\\
  \midrule
  Original & 0.855 & 403.9\\
  \midrule
  \acrshort{AAE} & 0.859 & 501.1\\
  \acrshort{Char-RNN} & 0.856 & 482.4\\
  Combinatorial & 0.873 & 536.9\\
  \acrshort{HMM} & 0.871 & 250.9\\
  \acrshort{JTN} & 0.856 & 489.5\\
  Latent \acrshort{GAN} & 0.857 & 486.4\\
  N-gram & 0.874 & 479.8\\
  \acrshort{VAE} & 0.856 & 475.3\\
  \bottomrule
\end{tabular}
\caption{
\label{tab:molecule_metrics}
IntDiv and \acrshort{VS} for generative models of molecules. The \acrshort{HMM} has one of the highest IntDiv scores, but scores much lower on \acrshort{VS} . 
An analysis of $250$ molecules from the \acrshort{HMM} reveals \acrshort{VS} is more accurate in this case. (See \ref{fig:molecule_kernels}.)
}
\end{table}
We evaluate samples from generative models provided in the \acrshort{MOSES} benchmark~\citep{polykovskiy2020moses}, using the first 2,500 valid molecules in each sample.
Following prior work, our similarity function is the Morgan fingerprint similarity (radius 2), implemented in RDKit.\footnote{RDKit: Open-source Cheminformatics. https://www.rdkit.org.}
IntDiv ranks the \acrshort{HMM} among the most diverse models, while ${\ourmetric}$ ranks it as the least diverse  (see Section~\ref{sec:molecule_models}).

\subsection{Evaluating image generative models for diversity}

In Table~\ref{tab:app_comparing_image_models}, we replicate the table described in Section~\ref{sec:images} and add an additional column, which evaluates diversity using the cosine similarity between pixel vectors as the similarity function.

\begin{table}[h]
  \small
  \setlength{\tabcolsep}{4pt}
    \begin{minipage}{.49\linewidth}
      \centering
      \resizebox{\linewidth}{!}{\begin{tabular}{l r r r r r r}
          \toprule
          Model & IS$_\uparrow$ & FID$_\downarrow$ & Prec$_\uparrow$ & Rec$_\uparrow$ & ${\ourmetric}_I$$_\uparrow$ & ${\ourmetric}_P$$_\uparrow$ \\
          \\
          \multicolumn{6}{l}{\textbf{CIFAR-10}} \\
          \midrule
          Original &  &  &  &  & 19.50 & 3.52\\
          VDVAE & 5.82 & 40.05 & 0.63 & 0.35 & 12.87 & 3.34\\
          DenseFlow & 6.01 & 34.54 & 0.62 & 0.38 & 13.55 & 2.94\\
          IDDPM & 9.24 & 4.39 & 0.66 & 0.60 & 16.86 & 3.27 \\
          \\
          \multicolumn{6}{l}{\textbf{ImageNet 64$\times$64}} \\
          \midrule
          Original &  &  &  &  & 43.93 & 4.43\\
          VDVAE & 9.68 & 57.57 & 0.47 & 0.37 & 18.04 & 4.24\\
          DenseFlow & 5.62 & 102.90 & 0.36 & 0.17 & 12.71 & 3.51\\
          IDDPM & 15.59 & 19.24 & 0.59 & 0.58 & 24.28 & 4.57\\
        \end{tabular}
     }
    \end{minipage}\hfill
    \begin{minipage}{.49\linewidth}
      \centering
      \resizebox{\linewidth}{!}{\begin{tabular}{l r r r r r r}
          \toprule
          Model & IS$_\uparrow$ & FID$_\downarrow$ & Prec$_\uparrow$ & Rec$_\uparrow$ & ${\ourmetric}_I$$_\uparrow$ & ${\ourmetric}_P$$_\uparrow$ \\
          \\
          \multicolumn{6}{l}{\textbf{LSUN Bedroom 256$\times$256}} \\
          \midrule
          Original &  &  &  &  & 8.99 & 3.10\\
          StyleGAN & 2.55 & 2.35 & 0.59 & 0.48 & 8.76 & 3.09\\
          ADM & 2.38 & 1.90 & 0.66 & 0.51 & 7.97 & 3.27\\
          RQ-VT & 2.56 & 3.16 & 0.60 & 0.50 & 8.48 & 3.67\\
          \\
          \multicolumn{6}{l}{\textbf{LSUN Cat 256$\times$256}} \\
          \midrule
          Original &  &  &  &  & 15.12 & 4.58\\
          StyleGAN2 & 4.84 & 7.25 & 0.58 & 0.43 & 13.55 & 4.53\\
          ADM & 5.19 & 5.57 & 0.63 & 0.52 & 13.09 & 4.81\\
          RQ-VT & 5.76 & 10.69 & 0.53 & 0.48 & 14.91 & 5.83\\
        \end{tabular}
        }
    \end{minipage} 
\caption{
\label{tab:app_comparing_image_models}
We evaluate samples from several recent models, measuring similarity using either Inception representations (${\ourmetric}_I$) or pixels (${\ourmetric}_P$).
The pixel similarity score is the cosine similarity between pixel vectors, calculated after resizing the images to 32$\times$32 pixels.
The pixel similarity and Inception similarity scores do not always agree---for example, if the images in a sample represent a variety of ImageNet classes by share a similar color palette, we might expect the sample to have high Inception diversity but low pixel diversity.
The pixel diversity scores are on a lower scale, indicating that this similarity metric is less capable of making fine-grained distinctions between the images in these samples.
}
\end{table}

\acrshort{VS} should be understood as the diversity with respect to a specific similarity function, in this case, the Inception ImageNet similarity.
We illustrate this point in Figure~\ref{fig:lsun_cat_comparing_kernels} by comparing the top eigenvalues of the kernel matrices corresponding to the 
Inception similarity and the pixel similarity, which we calculate by resizing the images to 32$\times$32 pixels and taking the cosine similarity between pixel vectors.
Inception similarity provides a form of semantic similarity, with components corresponding to particular cat breeds, while the pixel kernel provides a 
simple form of visual similarity, with components corresponding to broad differences in lightness, darkness, and color.
\begin{figure*}[t!]
     \centering
     \begin{subfigure}[c]{0.5\linewidth}
         \centering
         \includegraphics[width=\linewidth]{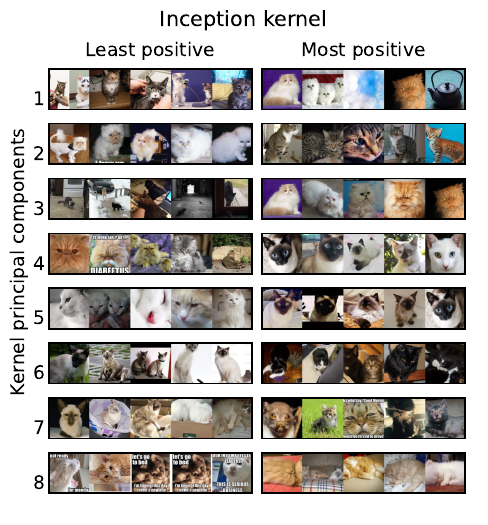}
     \end{subfigure}\hfill
     \begin{subfigure}[c]{0.5\linewidth}
         \centering
         \includegraphics[width=\linewidth]{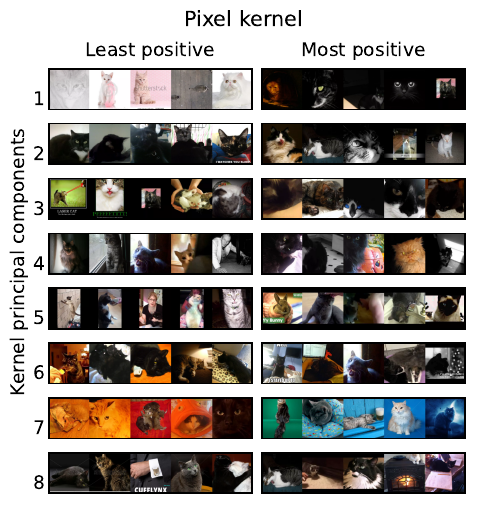}
     \end{subfigure}
\caption{
\label{fig:lsun_cat_comparing_kernels}
The choice of similarity function provides a way of specifying the notion of diversity that is relevant for a given application.
We project \acrshort{LSUN} Cat images along the top eigenvectors of the kernel matrix, using either Inception features or pixels to define similarity.
Inception similarity provides a form of semantic similarity, with components corresponding to particular cat breeds, while the pixel kernel captures visual similarity.
For each eigenvector $\vu$, we show the four images with the highest and lowest entries in $\vu$. 
For both kernels, every similarity score is positive, so all entries in the top eigenvector have the same sign; the images with the 
highest weights in this component have the highest expected similarity scores.
The remaining eigenvectors partition the images along different dimensions of variation.
}
\end{figure*}

\subsection{Evaluating decoding algorithms for text for diversity}
\label{app:text}

\begin{figure}[h]
\centering
\resizebox{0.5\linewidth}{!}{\includegraphics{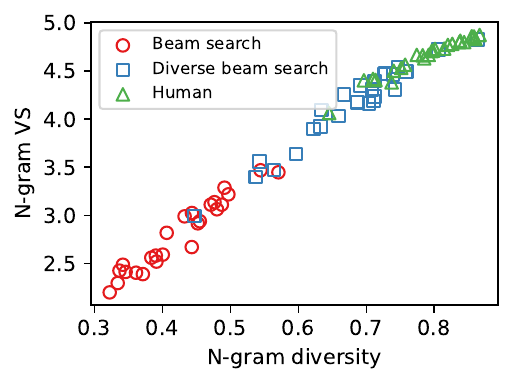}
}
\caption{
\label{fig:caption_diversity}
${\ourmetric}$ is correlated with N-gram diversity. Each point represents a group of five captions for a particular image.
}
\end{figure}

In Figure~\ref{fig:caption_diversity}, we plot the relationship between ${\ourmetric}$ and n-gram diversity using the MS-COCO captioning data and the n-gram overlap kernel described in Section~\ref{sec:text}.
The figure shows that ${\ourmetric}$ is highly correlated with n-gram diversity, which is expected given that our similarity function is based on n-gram overlap.
Nonetheless, there are some data points that the metrics rank differently.
This is because n-gram diversity conflates two properties: the diversity of n-grams within a single sentences and the n-gram overlap between sentences.
We highlight two examples in Figure~\ref{fig:caption_diversity_example}.
In general, the instances that n-gram diversity ranks lower compared to ${\ourmetric}$ contain individual sentences that repeat phrases.
On the other hand, n-gram diversity can be inflated in cases when one sentence in the sample is much longer than the others, even if the other sentences are not diverse.

\begin{figure}[h!]
  \small
\begin{minipage}[t]{.48\textwidth}
  High Vendi Score, low n-gram diversity:
  \it{}
  \begin{itemize}[leftmargin=*]
\item two men in bow ties standing next to steel rafter.
\item several men in suits talking together in a room.
\item an older \textbf{man in a tuxedo} standing next to a younger \textbf{man in a tuxedo} wearing glasses.
\item two men wearing tuxedos glance at each other.
\item older \textbf{man in tuxedo} sitting next to another younger \textbf{man in tuxedo}.
  \end{itemize}
\end{minipage}\hfill
\begin{minipage}[t]{.48\textwidth}
  Low Vendi Score, high n-gram diversity:
  \it{}
  \begin{itemize}[leftmargin=*]
\item a man and woman cutting a slice of cake by trees.
\item a couple of people standing cutting a cake.
\item \textbf{the dork with the earring stands next to the asian beauty who is way out of his league.}
\item a newly married couple cutting a cake in a park.
\item a bride and groom are cutting a cake as they smile.
\end{itemize}
\end{minipage}
\caption{
\label{fig:caption_diversity_example}
Two sets of captions that receive different ranks according Vendi Score and n-gram diversity.
We manually highlight some features contributing to the different scores. 
On the left, a sentence contains repeated n-grams, which are penalized by n-gram diversity.
On the right, one long outlier sentence contributes most of the n-grams for this group, greatly increasing the n-gram diversity.
}
\end{figure}

\subsection{Diagnosing datasets for diversity}
\label{app:diagnosing_datasets}

\paragraph{Molecules}
We evaluate the diversity scores of molecules in the GoodScents database of perfume materials,\footnote{http://www.thegoodscentscompany.com/} which has been used in prior machine learning research on odor modeling~\citep{sanchez2019machine}.
We use the standardized version of the data provided by the Pyrfume library.~\footnote{https://pyrfume.org/}
Each molecule in the dataset is labeled with one or more odor descriptors (for example, ``clean, oily, waxy'' or ``floral, fruity, green'').
We form groups of molecules corresponding to the seven most common odor descriptors, with each group consisting of 500 randomly sampled molecules.
We evaluate VS using two similarity functions: the Morgan fingerprint similarity (radius 2), and the similarity between odor descriptors, defined as the cosine similarity between descriptor indicator vectors $\phi(x)$, where $\phi_i(x)$ is equal to one if descriptor $i$ is associated with molecule $x$ and zero otherwise.

\begin{figure}[!hbpt]
\begin{minipage}{\linewidth}
     \centering
      \includegraphics{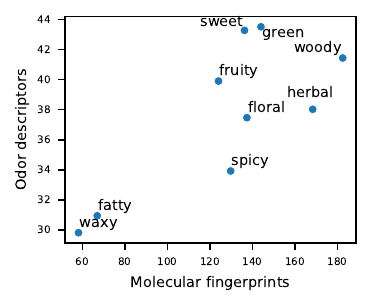}
\captionof{figure}{
\label{fig:goodscents_categories}
The Vendi Scores of samples containing $500$ molecules with different scent labels, calculating diversity using two similarity functions:
Morgan molecular fingerprint similarity, and the similarity between odor descriptors. Each molecule is associated 
with one or more human-written tags (e.g. ``floral, fruity, green, sweet''), and the odor-descriptor similarity is the cosine similarity between binary tag indicator vectors.
}
\end{minipage}
\end{figure}

The diversity scores are plotted in Figure~\ref{fig:goodscents_categories}.
The molecular diversity score and the odor-descriptor diversity scores are correlated, meaning that words like ``woody'' and ``green'' are used to describe molecules that vary in molecular structure and also elicit diverse odor descriptions, while words like ``waxy'' and ``fatty'' are used for molecules that are similar to each other and elicit similar odor descriptions.
For example, the word ``green'' appears in tag sets such as
``aldehydic, citrus, cortex, green, herbal, tart'' and ``floral, green, terpenic, tropical, vegetable, woody'', whereas the word ``waxy'' tends to co-occur with the same tags
(``fresh, waxy'';
``fresh, green, melon rind, mushroom, tropical, waxy'';
``fruity, green, musty, waxy'').
Molecules from the categories with the highest and lowest scores are illustrated in Figure~\ref{fig:molecule_categories}.

\begin{figure}[ht!]
\begin{minipage}{\linewidth}
     \centering
      \includegraphics[width=\linewidth]{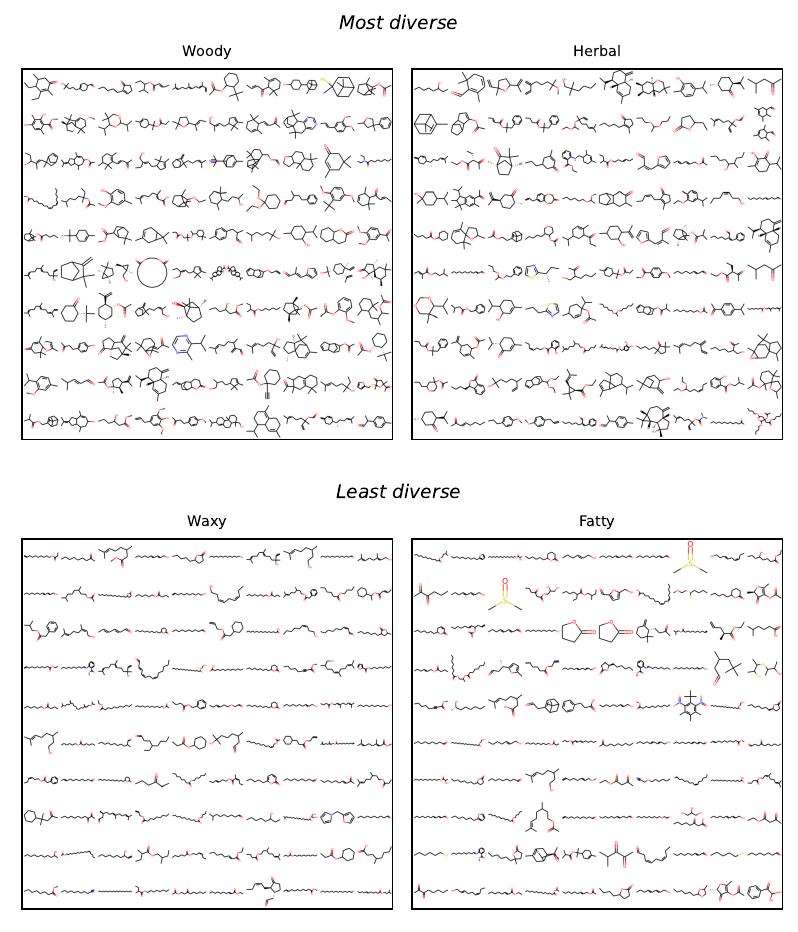}
\captionof{figure}{
\label{fig:molecule_categories}
The scent categories in Goodscents dataset with the highest (top) and lowest (bottom) \gls{VS}, using the molecular fingerprint similarity.
We show $100$ examples from each category, in decreasing order of average similarity, with the image at the top left having the highest average similarity scores.
}
\end{minipage}
\end{figure}

\begin{minipage}{\linewidth}
     \centering
     \begin{minipage}{0.5\linewidth}
      \includegraphics[width=\linewidth]{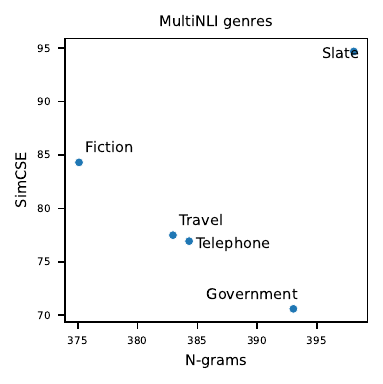}
      \label{subfig:mnli_categories}
     \end{minipage}\begin{minipage}{0.5\linewidth}
      \includegraphics[width=\linewidth]{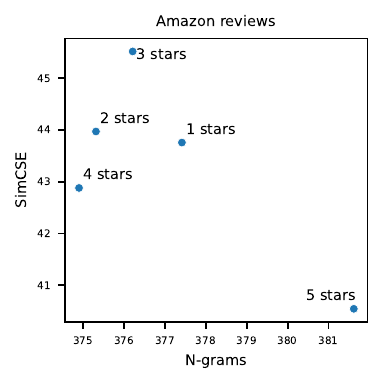}
      \label{subfig:amazon_categories}
     \end{minipage}
\captionof{figure}{
\label{fig:text_categories}
The Vendi Scores of samples containing 500 MultiNLI sentences with different genres (left) or Amazon reviews with different star ratings (right), defining similarity using either n-gram overlap or SimCSE~\citep{gao2021simcse}.
}
\end{minipage}

\paragraph{Text}
In Figure~\ref{fig:text_categories}, we evaluate the diversity scores of samples sentences with different genres, from the MultiNLI dataset~\citep{williams2018broad}, and Amazon product reviews with different star ratings~\citep{keung2020multilingual}, using either the n-gram overlap similarity or SimCSE~\citep{gao2021simcse}.
SimCSE is a Transformer-based sentence encoder that achieves state-of-the-art scores on semantic similarity benchmarks.
The model we use initialized from the uncased BERT-base model~\citep{devlin-etal-2019-bert} and trained with a contrastive learning objective to assign high similarity scores to pairs of MultiNLI sentences that have a logical entailment relationship.

In MultiNLI, both models assign the highest score to Slate, which consists of sentences from articles published on slate.com.
SimCSE assigns a higher score to the ``Fiction'' category, possibly because it is less sensitive to common n-grams (e.g. ``he said''), that appear in many sentences in this genre and contribute to the low N-gram diversity score.
In the Amazon review dataset, the 5-star reviews have the highest N-gram diversity but the lowest SimCSE diversity, perhaps because SimCSE assigns high similarity scores to sentences that have the same strong sentiment.
SimCSE assigns the highest diversity score to 3-star reviews, which can vary in sentiment.

\paragraph{Images}
Following the setting in~\ref{sec:diagnosing_datasets}, we evaluate two additional dataset, Fashion \acrshort{MNIST}~\citep{Xiao2017} and \acrshort{CelebA}~\citep{Liu2015}. 
We use the same similarity scores as in~\ref{sec:diagnosing_datasets}.
Images in CelebA are associated with 40-dimensional binary attribute vectors.
We use these attributes as an additional similarity score, defining the attribute similarity as the cosine similarity between attribute vectors.
These illustrations highlight the importance of the choice of similarity function in defining a diversity metric.

\begin{figure}[h]
     \centering
     \begin{subfigure}[c]{\linewidth}
      \includegraphics[width=\linewidth]{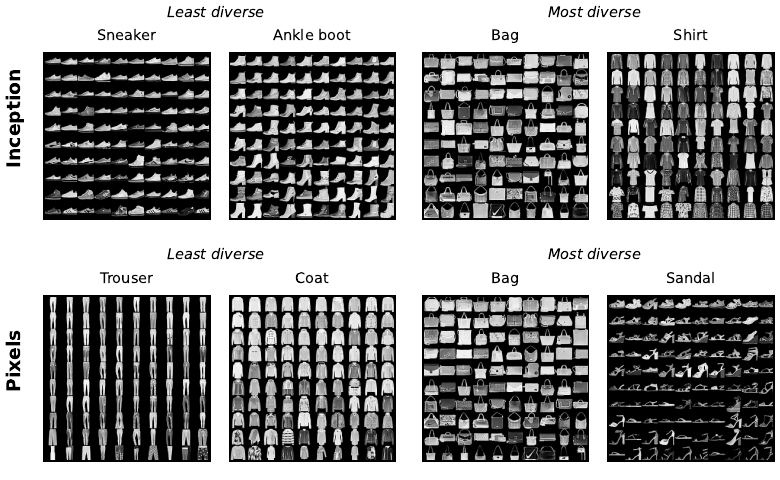}
      \label{subfig:fashion_mnist_categories}
     \end{subfigure}
\caption{
\label{fig:fashion_mnist_categories}
The categories in Fashion MNIST with the lowest (left) and highest (right) Vendi Scores, defining similarity as the cosine similarity between either Inception embeddings (top) or pixel vectors (bottom).
We show 100 examples from each category, in decreasing order of average similarity, with the image at the top left having the highest average similarity scores according to the corresponding kernel.
}
\end{figure}

\begin{figure}[!hbpt]
     \centering
      \includegraphics[width=\linewidth]{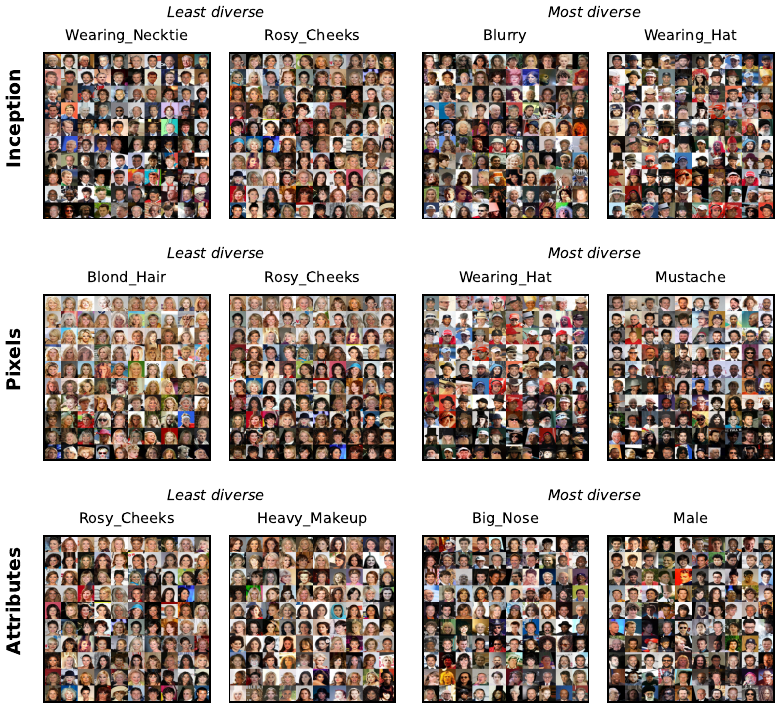}
\caption{
\label{fig:celeba_categories}
The attributes in \acrshort{CelebA} with the lowest (left) and highest (right) \acrshort{VS}, defining similarity as the cosine similarity 
between either Inception embeddings (top), pixel vectors (middle), or binary attribute vectors (bottom).
We show $100$ examples from each category, in decreasing order of average similarity, with the image at the top left having the highest average similarity scores according to the corresponding kernel.
These examples illustrate the importance of the choice of similarity function for defining the notion of diversity that is relevant for a given application. However, almost all choices of similarity functions show that the \acrshort{CelebA} dataset is more diverse for men than for women. 
}
\end{figure}

\end{document}